\documentclass{article}
\usepackage{nips12submit_e,times}

\usepackage{bm}
\usepackage{amsthm} 
\usepackage{amsmath}
\usepackage{amssymb}
\usepackage{multirow}
\usepackage{booktabs}
\usepackage{mathrsfs} 
\usepackage{paralist}
\usepackage{rotating}
\usepackage{wrapfig} 
\usepackage{caption}
\usepackage{subcaption}
\usepackage{appendix}

\usepackage{graphicx} % more modern

\usepackage{hyperref}
 
\newtheorem{mydef}{Definition}
\newtheorem{mythm}[mydef]{Theorem}
\newtheorem{mylem}[mydef]{Lemma}

\newcommand{\inspace}{\ensuremath{\mathcal{X}}}   % the input space X
\newcommand{\pp}[1]{\ensuremath{\mathbb{#1}}}     % probability measure
\newcommand{\pspace}{\ensuremath{\mathscr{P}}}    % probability space
\newcommand{\hbspace}{\ensuremath{\mathcal{H}}}   % Hilbert space

\newcommand{\abbrvmm}[1]{\ensuremath{\mu_{#1}}}

\newcommand{\dd}{\, \mathrm{d}} 

\providecommand{\abs}[1]{\lvert#1\rvert}
\providecommand{\norm}[1]{\lVert#1\rVert}

\nipsfinalcopy

\title{Learning from Distributions via Support Measure Machines}

\author{Krikamol Muandet \\
  MPI for Intelligent Systems, 
  T\"{u}bingen \\
  \href{mailto:krikamol@tuebingen.mpg.de}{\texttt{krikamol@tuebingen.mpg.de}}
  \And
  Kenji Fukumizu \\
  The Institute of Statistical Mathematics, Tokyo \\  
  \href{mailto:fukumizu@ism.ac.jp}{\texttt{fukumizu@ism.ac.jp}}
  \AND
  Francesco Dinuzzo \\
  MPI for Intelligent Systems, T\"{u}bingen \\
  \href{mailto:fdinuzzo@tuebingen.mpg.de}{\texttt{fdinuzzo@tuebingen.mpg.de}}
  \And
  Bernhard Sch\"{o}lkopf \\
  MPI for Intelligent Systems, 
  T\"{u}bingen \\
  \href{mailto:bs@tuebingen.mpg.de}{\texttt{bs@tuebingen.mpg.de}}
}

\begin{document} 
\maketitle

\begin{abstract} 
  This paper presents a kernel-based discriminative learning framework on probability measures. Rather than relying on large collections of vectorial training examples, our framework learns using a collection of probability distributions that have been constructed to meaningfully represent training data. By representing these probability distributions as mean embeddings in the reproducing kernel Hilbert space (RKHS), we are able to apply many standard kernel-based learning techniques in straightforward fashion. To accomplish this, we construct a generalization of the support vector machine (SVM) called a support measure machine (SMM). Our analyses of SMMs provides several insights into their relationship to traditional SVMs. Based on such insights, we propose a flexible SVM (Flex-SVM) that places different kernel functions on each training example. Experimental results on both synthetic and real-world data demonstrate the effectiveness of our proposed framework.
\end{abstract}

\section{Introduction}

Discriminative learning algorithms are typically trained from large collections of vectorial training examples. In many classical learning problems, however, it is arguably more appropriate to represent training data not as individual data points, but as probability distributions. There are, in fact, 
multiple reasons why probability distributions may be preferable.

Firstly, uncertain or missing data naturally arises in many applications. For example, gene expression data obtained from the microarray experiments are known to be very noisy due to various sources of variabilities \cite{Yang02:cDNA}. In order to reduce uncertainty, and to allow for estimates of confidence levels, experiments are often replicated. Unfortunately, the feasibility of replicating the microarray experiments is often inhibited by cost constraints, as well as the amount of available mRNA.
To cope with experimental uncertainty given a limited amount of data, it is natural to represent each array as a probability distribution that has been designed to approximate the variability of gene expressions across slides.

Probability distributions may be equally appropriate given an abundance of training data. In data-rich disciplines such as neuroinformatics, climate informatics, and astronomy, a high throughput experiment can easily generate a huge amount of data, leading to significant computational challenges in both time and space. Instead of scaling up one's learning algorithms, one can scale down one's dataset by constructing a smaller collection of distributions which represents groups of similar samples. Besides computational efficiency, aggregate statistics can potentially incorporate higher-level information that represents the collective behavior of multiple data points.

Previous attempts have been made to learn from distributions by creating positive definite (p.d.) kernels on probability measures. In \cite{Jebara04Probabilityproduct}, the probability product kernel (PPK) was proposed as a generalized inner product between two input objects, which is in fact closely related to well-known kernels such as the Bhattacharyya kernel \cite{Bhattacharyya43Kernel} and the exponential symmetrized Kullback-Leibler (KL) divergence \cite{Moreno04KL}. In \cite{Hein05Hilbertian}, an extension of a two-parameter family of Hilbertian metrics of Tops\o e was used to define Hilbertian kernels on probability measures. In \cite{Cuturi05SKM}, the semi-group kernels were designed for objects with additive semi-group structure such as positive measures. Recently, \cite{Martins09NIT} introduced nonextensive information theoretic kernels on probability measures based on new Jensen-Shannon-type divergences. Although these kernels have proven successful in many applications, they are designed specifically for certain properties of distributions and application domains. Moreover, there has been no attempt in making a connection to the kernels on corresponding input spaces.

The contributions of this paper can be summarized as follows. First, we prove the representer theorem for a regularization framework over the space of probability distributions, which is a generalization of regularization over the input space on which the distributions are defined (Section \ref{sec:regularization}). Second, a family of positive definite kernels on distributions is introduced (Section \ref{sec:distkernel}). Based on such kernels, a learning algorithm on probability measures called \emph{support measure machine} (SMM) is proposed. An SVM on the input space is provably a special case of the SMM. Third, the paper presents the relations between sample-based and distribution-based methods (Section \ref{sec:relation}). If the distributions depend only on the locations in the input space, the SMM particularly reduces to a more flexible SVM that places different kernels on each data point.

%The remainder of this paper is organized as follows. Section \ref{sec:regularization} introduces a regularization framework on probability measures. A family of kernels on probability distributions is presented in Section \ref{sec:distkernel}. Theoretical analyses are subsequently presented in Section \ref{sec:relation}, followed by a summary of related works in Section \ref{sec:relatedworks}. In Section \ref{sec:experiments}, experimental results on both synthetic and real-world datasets are presented with discussions. Finally, we conclude the paper with Section \ref{sec:conclusions}.

\section{Regularization on probability distributions} 
\label{sec:regularization}

Given a non-empty set $\inspace$, let $\pspace$ denote the set of all probability measures $\pp{P}$ on a measurable space $(\inspace,\mathcal{A})$, where $\mathcal{A}$ is a $\sigma$-algebra of subsets of $\inspace$. The goal of this work is to learn a function $h:\pspace\rightarrow\mathcal{Y}$ given a set of example pairs $\{(\pp{P}_i,y_i)\}_{i=1}^m$, where $\pp{P}_i\in\pspace$ and $y_i\in\mathcal{Y}$. In other words, we consider a supervised setting in which input training examples are probability distributions. In this paper, we focus on the binary classification problem, i.e., $\mathcal{Y}=\{+1,-1\}$.

In order to learn from distributions, we employ a compact representation that not only preserves necessary information of individual distributions, but also permits efficient computations. That is, we adopt a Hilbert space embedding to represent the distribution as a mean function in an RKHS \cite{Bertinet04:RKHS, Smola07Hilbert}. Formally, let $\hbspace$ denote an RKHS of functions $f:\inspace\rightarrow\mathbb{R}$, endowed with a reproducing kernel $k:\inspace\times\inspace\rightarrow\mathbb{R}$. The mean map from $\pspace$ into $\hbspace$ is defined as
\begin{equation}
\label{eq:meanmap}
\mu : \pspace \rightarrow \hbspace, \enspace \pp{P} \longmapsto \int_{\inspace}k(x,\cdot)\dd
\pp{P}(x) \enspace .
\end{equation}
We assume that $k(x,\cdot)$ is bounded for any $x\in\inspace$. It can be shown that, if $k$ is characteristic, the map \eqref{eq:meanmap} is injective, i.e., all the information about the distribution is preserved \cite{Sriperumbudur10:Metrics}. For any $\pp{P}$, letting $\abbrvmm{\pp{P}}=\mu(\pp{P})$, we have the reproducing property
\begin{equation}
\mathbb{E}_{\pp{P}}[f]=\langle\abbrvmm{\pp{P}},f \rangle_{\hbspace}, \enspace \forall f\in\hbspace
\enspace .
\end{equation} 
That is, we can see the mean embedding $\abbrvmm{\pp{P}}$ as a feature map associated with the kernel $K:\pspace\times\pspace\rightarrow\mathbb{R}$, defined as $K(\pp{P},\pp{Q})=\langle\abbrvmm{\pp{P}},\abbrvmm{\pp{Q}}\rangle_{\hbspace}$. Since $\sup_x\|k(x,\cdot)\|_{\hbspace}<\infty$, it also follows that $K(\pp{P},\pp{Q}) = \iint\langle k(x,\cdot),k(z,\cdot)\rangle_{\hbspace}\dd\pp{P}(x)\dd\pp{Q}(z)=\iint k(x,z)\dd\pp{P}(x)\dd\pp{Q}(z)$, where the second equality follows from the reproducing property of $\hbspace$. It is immediate that $K$ is a p.d. kernel on $\pspace$.

The following theorem shows that optimal solutions of a suitable class of regularization problems involving distributions can be expressed as a finite linear combination of mean embeddings.

%% representer theorem for probability measures
\begin{mythm}
\label{thm:representer}
Given training examples $(\pp{P}_i,y_i)\in \pspace\times\mathbb{R},\,i=1,\dotsc,m$, a strictly monotonically increasing function $\Omega:[0,+\infty)\rightarrow\mathbb{R}$, and a loss function $\ell:(\pspace\times\mathbb{R}^2)^m \rightarrow\mathbb{R}\cup\{+\infty\}$, any $f\in\mathcal{H}$ minimizing the regularized risk functional
\begin{equation}
  \label{eq:regfunc} 
  \ell\left(\pp{P}_1,y_1,\mathbb{E}_{\pp{P}_1}[f],\dotsc,\pp{P}_m,y_m,\mathbb{E}_{\pp{P}_m}[f]\right) 
  + \Omega\left(\|f\|_{\mathcal{H}}\right)
\end{equation}
\noindent admits a representation of the form $ f = \sum_{i=1}^m\alpha_i \abbrvmm{\pp{P}_i}$ for some $\alpha_i\in\mathbb{R}, \, i=1,\dotsc,m$.
\end{mythm}
 
%% proof of representer theorem
%\begin{proof}
%  By virtue of Proposition 2 in \cite{Sriperumbudur10:Metrics}, the linear functional $\mathbb{E}_{\pp{P}}[\cdot]$ are bounded for all $\pp{P}\in\pspace$. Then, given $\pp{P}_1,\pp{P}_2,...,\pp{P}_m$, any $f\in\mathcal{H}$ can be decomposed as $f = f_{\mu} + f^{\perp}$ where $f_{\mu}\in\mathcal{H}$ lives in the span of $\abbrvmm{\pp{P}_i}$, i.e., $f_{\mu}=\sum_{i=1}^m\alpha_i\abbrvmm{\pp{P}_i}$ and $f^{\perp}\in\mathcal{H}$ satisfying, for all $j$, $\langle f^{\perp},\abbrvmm{\pp{P}_j}\rangle = 0$. Hence, for all $j$, we have $\mathbb{E}_{\pp{P}_j}[f] = \mathbb{E}_{\pp{P}_j}[f_{\mu} + f^{\perp}] = \langle f_{\mu}+f^{\perp},\abbrvmm{\pp{P}_j}\rangle = \langle f_{\mu},\abbrvmm{\pp{P}_j}\rangle + \langle f^{\perp},\abbrvmm{\pp{P}_j}\rangle = \langle f_{\mu},\abbrvmm{\pp{P}_j}\rangle$ which is independent of $f^{\perp}$. As a result, the loss functional $\ell$ in \eqref{eq:regfunc} does not depend on $f^{\perp}$. For the regularization functional $\Omega$, since $f^{\perp}$ is orthogonal to $\sum_{i=1}^m\alpha_i\abbrvmm{\pp{P}_i}$ and $\Omega$ is strictly monotonically increasing, we have $\Omega(\|f\|) = \Omega(\|f_{\mu} + f^{\perp}\|)=\Omega(\sqrt{\|f_{\mu}\|^2 + \|f^{\perp}\|^2})\geq \Omega(\|f_{\mu}\|)$ with equality if and only if $f^{\perp}=0$ and thus $f=f_{\mu}$. Consequently, any minimizer must take the form $f=\sum_{i=1}^m\alpha_i\abbrvmm{\pp{P}_i} = \sum_{i=1}^m\alpha_i\mathbb{E}_{\pp{P}_i}[k(x,\cdot)]$. 
%\end{proof}

Theorem \ref{thm:representer} clearly indicates how each distribution contributes to the minimizer of \eqref{eq:regfunc}. Roughly speaking, the coefficients $\alpha_i$ controls the contribution of the distributions through the mean embeddings $\abbrvmm{\pp{P}_i}$. Furthermore, if we restrict $\pspace$ to a class of Dirac measures $\delta_x$ on $\inspace$ and consider the training set $\{(\delta_{x_i},y_i)\}_{i=1}^m$, the functional \eqref{eq:regfunc} reduces to the usual regularization functional \cite{Schoelkopf01:Representer} and the solution reduces to $f=\sum_{i=1}^m\alpha_ik(x_i,\cdot)$. Therefore, the standard representer theorem is recovered as a particular case (see also \cite{Dinuzzo12:representer} for more general results on representer theorem).

Note that, on the one hand, the minimization problem \eqref{eq:regfunc} is different from minimizing the functional $\mathbb{E}_{\pp{P}_1}\dotsc\mathbb{E}_{\pp{P}_m}\ell(x_1,y_1,f(x_1),\dotsc,x_m,y_m,f(x_m)) + \Omega(\|f\|_{\hbspace})$ for the special case of the additive loss $\ell$. Therefore, the solution of our regularization problem is different from what one would get in the limit by training on an infinitely many points sampled from $\pp{P}_1,\dotsc,\pp{P}_m$. On the other hand, it is also different from minimizing the functional $\ell(M_1,y_1,f(M_1),\dotsc,M_m,y_m,f(M_m))+\Omega(\|f\|_{\hbspace})$ where $M_i = \mathbb{E}_{x\sim\pp{P}_i}[x]$. In a sense, our framework is something in between.

\section{Kernels on probability distributions} 
\label{sec:distkernel} 

As the map \eqref{eq:meanmap} is linear in $\pspace$, optimizing the functional \eqref{eq:regfunc} amounts to finding a function in $\hbspace$ that approximate well functions from $\pspace$ to $\mathbb{R}$ in the function class $\mathcal{F} \triangleq \{\pp{P} \rightarrow \int_{\inspace}g\dd\pp{P}\,| \,\pp{P}\in\pspace,\,g\in C(\inspace)\}$ where $C(\inspace)$ is a class of bounded continuous functions on $\inspace$. Since $\delta_x\in\pspace$ for any $x\in\inspace$, it follows that $C(\inspace)\subset \mathcal{F} \subset C(\pspace)$ where $C(\pspace)$ is a class of bounded continuous functions on $\pspace$ endowed with the topology of weak convergence and the associated Borel $\sigma$-algebra. The following lemma states the relation between the RKHS $\hbspace$ induced by the kernel $k$ and the function class $\mathcal{F}$.

%% universality 
\begin{mylem}
  \label{lem:universal-linear}
  Assuming that $\inspace$ is compact, the RKHS $\hbspace$ induced by a kernel $k$ is dense in $\mathcal{F}$ if $k$ is universal, i.e., for every function $F\in\mathcal{F}$ and every $\varepsilon > 0$ there exists a function $g\in\hbspace$ with $\sup_{\pp{P}\in\pspace}\abs{F(\pp{P}) - \int g \dd\pp{P}} \leq \varepsilon$.
\end{mylem}
\vspace{-5pt}
\begin{proof}
  Assume that $k$ is universal. Then, for every function $f\in C(\inspace)$ and every $\varepsilon >0$ there exists a function $g\in\hbspace$ induced by $k$ with $\sup_{x\in\inspace}\abs{f(x)-g(x)}\leq\varepsilon$ \cite{Steinwart01universal}. Hence, by linearity of $\mathcal{F}$, for every $F\in\mathcal{F}$ and every $\varepsilon > 0$ there exists a function $h\in\hbspace$ such that $\sup_{\pp{P}\in\pspace}\abs{F(\pp{P}) - \int h \dd\pp{P}} \leq \varepsilon$.
\end{proof}
\vspace{-5pt}

Nonlinear kernels on $\pspace$ can be defined in an analogous way to nonlinear kernels on $\inspace$, by treating mean embeddings $\abbrvmm{\pp{P}}$ of $\pp{P}\in\pspace$ as its feature representation. First, assume that the map \eqref{eq:meanmap} is injective and let $\langle\cdot,\cdot\rangle_{\pspace}$ be an inner product on $\pspace$. By linearity, we have $\langle\pp{P},\pp{Q}\rangle_{\pspace} = \langle\abbrvmm{\pp{P}},\abbrvmm{\pp{Q}}\rangle_{\hbspace}$ (cf. \cite{Bertinet04:RKHS} for more details). Then, the nonlinear kernels on $\pspace$ can be defined as $K(\pp{P},\pp{Q})= \kappa(\abbrvmm{\pp{P}},\abbrvmm{\pp{Q}}) = \langle\psi(\abbrvmm{\pp{P}}),\psi(\abbrvmm{\pp{Q}})\rangle_{\hbspace_{\kappa}}$ where $\kappa$ is a p.d. kernel. As a result, many standard nonlinear kernels on $\inspace$ can be used to define nonlinear kernels on $\pspace$ as long as the kernel evaluation depends entirely on the inner product $\langle\abbrvmm{\pp{P}},\abbrvmm{\pp{Q}}\rangle_{\hbspace}$, e.g., $K(\pp{P},\pp{Q})=(\langle\abbrvmm{\pp{P}},\abbrvmm{\pp{Q}}\rangle_{\hbspace} + c)^d$. Although requiring more computational effort, their practical use is simple and flexible. Specifically, the notion of p.d. kernels on distributions proposed in this work is so generic that standard kernel functions can be reused to derive kernels on distributions that are different from many other kernel functions proposed specifically for certain distributions.

It has been recently proved that the Gaussian RBF kernel given by $K(\pp{P},\pp{Q}) = \exp(-\frac{\gamma}{2}\|\abbrvmm{\pp{P}}-\abbrvmm{\pp{Q}}\|_{\hbspace}^2), \; \forall \pp{P},\pp{Q}\in\pspace$ is universal w.r.t $C(\pspace)$ given that $\inspace$ is compact and the map $\mu$ is injective \cite{Christmann10Universal}. Despite its success in real-world applications, the theory of kernel-based classifiers beyond the input space $\inspace\subset\mathbb{R}^d$, as also mentioned by \cite{Christmann10Universal}, is still incomplete. It is therefore of theoretical interest to consider more general classes of universal kernels on probability distributions.

\subsection{Support measure machines}

This subsection extends SVMs to deal with probability distributions, leading to \emph{support measure machines} (SMMs). In its general form, an SMM amounts to solving an SVM problem with the expected kernel $K(\pp{P},\pp{Q}) = \mathbb{E}_{x\sim\pp{P},z\sim\pp{Q}}[k(x,z)]$. This kernel can be computed in closed-form for certain classes of distributions and kernels $k$. Examples are given in Table 
\ref{tab:expected-kernel}. 

%%%%%%%%%%%%%%%%%%%%%%%%%%%%%%%%%%%%%%%%%%%%%%%%%%%%%%%%%%%%
%% the close-form solutions of expected kernels
\begin{table*}[t!]
  \centering
  \caption{the analytic forms of expected kernels for different choices of kernels and distributions.}
  \resizebox{\linewidth}{!}{
  \begin{tabular}{lll} 
    \toprule
    \textbf{Distributions} & \textbf{Embedding kernel} $k(x,y)$ & $K(\pp{P}_i,\pp{P}_j)=\langle\abbrvmm{\pp{P}_i},\abbrvmm{\pp{P}_j}\rangle_{\hbspace}$\\
    \midrule
    Arbitrary $\pp{P}(m;\Sigma)$ & Linear $\langle x,y \rangle$ & $m_i^{\mathsf{T}}m_j + \delta_{ij}\text{tr }
    \Sigma_i$ \\
    Gaussian $\mathcal{N}(m;\Sigma)$ & Gaussian RBF $\exp(-\frac{\gamma}{2}\|x-y\|^2)$ 
    & $\exp(-\frac{1}{2}(m_i-m_j)^{\mathsf{T}}(\Sigma_i+\Sigma_j + \gamma^{-1}\mathbf{I})^{-1}
    (m_i-m_j)) $ \\
    && $/\abs{\gamma\Sigma_i + \gamma\Sigma_j + \mathbf{I}}^{\frac{1}{2}}$\\
    Gaussian $\mathcal{N}(m;\Sigma)$ & Polynomial degree 2 $(\langle x,y\rangle + 1)^2$ 
    &  $(\langle m_i,m_j\rangle +1)^2 + \text{tr }\Sigma_i\Sigma_j + m_i^{\mathsf{T}}\Sigma_jm_i
    + m_j^{\mathsf{T}}\Sigma_im_j$ \\ 
    Gaussian $\mathcal{N}(m;\Sigma)$ & Polynomial degree 3 $(\langle x,y\rangle + 1)^3$ 
    & $(\langle m_i,m_j\rangle +1)^3 + 6m_i^{\mathsf{T}}\Sigma_i\Sigma_jm_j $ \\
    && $+ 3(\langle m_i,m_j\rangle +1)(\text{tr }\Sigma_i\Sigma_j + m_i^{\mathsf{T}}\Sigma_jm_i + m_j^{\mathsf{T}}\Sigma_im_j)$ \\
    \bottomrule
  \end{tabular}}
  \label{tab:expected-kernel}  
\end{table*}
%%%%%%%%%%%%%%%%%%%%%%%%%%%%%%%%%%%%%%%%%%%%%%%%%%%%%%%%%%%%
 
Alternatively, one can approximate the kernel $K(\pp{P},\pp{Q})$ by the empirical estimate:
\begin{equation}
  \label{eq:emp-kernel}
  K_{\text{emp}}(\widehat{\mathbb{P}}_n,\widehat{\mathbb{Q}}_m) 
  = \frac{1}{n\cdot m}\sum_{i=1}^n\sum_{j=1}^mk(x_i,z_j)
\end{equation}
\noindent where $\widehat{\mathbb{P}}_n$ and $\widehat{\mathbb{Q}}_m$ are empirical distributions of $\pp{P}$ and $\pp{Q}$ given random samples $\{x_i\}_{i=1}^n$ and $\{z_j\}_{j=1}^m$, respectively. A finite sample of size $m$ from a distribution $\pp{P}$ suffices (with high probability) to compute an approximation within an error of $O(m^{-\frac{1}{2}})$. Instead, if the sample set is sufficiently large, one may choose to approximate the true distribution by simpler probabilistic models, e.g., a mixture of Gaussians model, and choose a kernel $k$ whose expected value admits an analytic form. Storing only the parameters of probabilistic models may save some space compared to storing all data points. 

%Table \ref{tab:smm-svm} depicts the general forms of linear and nonlinear kernels for sample-based and distribution-based methods. 

Note that the standard SVM feature map $\phi(x)$ is usually nonlinear in $x$, whereas $\abbrvmm{\pp{P}}$ is \emph{linear} in $\pp{P}$. Thus, for an SMM, the first level kernel $k$ is used to obtain a vectorial representation of the measures, and the second level kernel $K$ allows for a nonlinear algorithm on distributions. For clarity, we will refer to $k$ and $K$ as the \textbf{embedding kernel} and the \textbf{level-2 kernel}, respectively

%Note that linear (resp. nonlinear) SVM and linear (resp. nonlinear) SMM use different forms of kernel functions, i.e., $x^{\mathsf{T}}x'$, $\langle\phi(x),\phi(x')\rangle_{\hbspace}$, $\langle\abbrvmm{\pp{P}},\abbrvmm{\pp{Q}}\rangle_{\hbspace}$, and $\langle\psi(\abbrvmm{\pp{P}}),\psi(\abbrvmm{\pp{Q}})\rangle_{\hbspace_{\tilde{k}}}$, respectively. 

\section{Theoretical analyses}
\label{sec:relation} 

This section presents key theoretical aspects of the proposed framework, which reveal important connection between kernel-based learning algorithms on the space of distributions and on the input space on which they are defined. 

\subsection{Risk deviation bound}

Given a training sample $\{(\pp{P}_i,y_i)\}_{i=1}^m$ drawn i.i.d. from some unknown probability distribution $\mathcal{P}$ on $\pspace\times\mathcal{Y}$, a loss function $\ell:\mathbb{R}\times\mathbb{R}\rightarrow\mathbb{R}$, and a function class $\Lambda$, the goal of statistical learning is to find the function $f\in\Lambda$ that minimizes the expected risk functional $\mathcal{R}(f)=\int_{\pspace}\int_{\inspace}\ell(y,f(x))\dd\pp{P}(x)\dd\mathcal{P}(\pp{P},y)$. Since $\mathcal{P}$ is unknown, the empirical risk $\mathcal{R}_{\text{emp}}(f)=\frac{1}{m}\sum_{i=1}^m\int_{\inspace}\ell(y_i,f(x))\dd\pp{P}_i(x)$ based on the training sample is considered instead. Furthermore, the risk functional can be simplified further by considering $\frac{1}{m\cdot n}\sum_{i=1}^m\sum_{x_{ij}\sim\pp{P}_i}\ell(y_i,f(x_{ij}))$ based on $n$ samples $x_{ij}$ drawn from each $\pp{P}_i$.

Our framework, on the other hand, alleviates the problem by minimizing the risk functional $\mathcal{R}^{\mu}(f) = \int_{\pspace} \ell(y,\mathbb{E}_{\pp{P}}[f(x)])\dd\mathcal{P}(\pp{P},y)$ for $f\in\hbspace$ with corresponding empirical risk functional $\mathcal{R}^{\mu}_{\text{emp}}(f)=\frac{1}{m}\sum_{i=1}^m\ell(y_i,\mathbb{E}_{\pp{P}_i}[f(x)])$ (cf. the discussion at the end of Section \ref{sec:regularization}). It is often easier to optimize $\mathcal{R}^{\mu}_{\text{emp}}(f)$ as the expectation can be computed exactly for certain choices of $\pp{P}_i$ and $\hbspace$. Moreover, for universal $\hbspace$, this simplification preserves all information of the distributions. Nevertheless, there is still a loss of information due to the loss function $\ell$.

Due to the i.i.d. assumption, the analysis of the difference between $\mathcal{R}$ and $\mathcal{R}^{\mu}$ can be simplified w.l.o.g. to the analysis of the difference between $\mathbb{E}_{\pp{P}}[\ell(y,f(x))]$ and $\ell(y,\mathbb{E}_{\pp{P}}[f(x)])$ for a particular distribution $\pp{P}\in\pspace$. The theorem below provides a bound on the difference between $\mathbb{E}_{\pp{P}}[\ell(y,f(x))]$ and $\ell(y,\mathbb{E}_{\pp{P}}[f(x)])$.

%% the table summarizes the distinction between SVM and SMM
%\begin{table}
%  \centering
%  \caption{The associated kernels for sample-based and distribution-based methods.}
%  \resizebox{8.2cm}{!}{
%    \scriptsize{
%  \begin{tabular}{lll}
%    \toprule
%    & Linear kernel & Nonlinear kernel \\
%    \midrule
%    Sample & $x^{\mathsf{T}}x'$ & $\langle\phi(x),\phi(x')\rangle_{\hbspace}$ \\
%    Distribution  & $\langle\abbrvmm{\pp{P}},\abbrvmm{\pp{Q}}\rangle_{\hbspace}$ 
%    & $\langle\psi(\abbrvmm{\pp{P}}),\psi(\abbrvmm{\pp{Q}})\rangle_{\hbspace_{\tilde{k}}}$ \\
%    \bottomrule
%  \end{tabular}}}
%  \label{tab:smm-svm}
%\end{table}

%% the risk deviation bound
\begin{mythm}
  \label{thm:deviation}
  Given an arbitrary probability distribution $\pp{P}$ with variance $\sigma^2$, a Lipschitz continuous function $f:\mathbb{R}\rightarrow\mathbb{R}$ with constant $C_f$, an arbitrary loss function $\ell : \mathbb{R}\times\mathbb{R}\rightarrow\mathbb{R}$ that is Lipschitz continuous in the second argument with constant $C_{\ell}$, it follows that $\abs{\mathbb{E}_{x\sim\pp{P}}[\ell(y,f(x))]  - \ell(y,\mathbb{E}_{x\sim\pp{P}}[f(x)])} \leq 2C_{\ell}C_f\sigma$ for any $y\in\mathbb{R}$.
\end{mythm}

Theorem \ref{thm:deviation} indicates that if the random variable $x$ is concentrated around its mean and the function $f$ and $\ell$ are well-behaved, i.e., Lipschitz continuous, then the loss deviation $|\mathbb{E}_{\pp{P}}[\ell(y,f(x))] - \ell(y,\mathbb{E}_{\pp{P}}[f(x)])|$ will be small. As a result, if this holds for any distribution $\pp{P}_i$ in the training set $\{(\pp{P}_i,y_i)\}_{i=1}^m$, the true risk deviation $\abs{\mathcal{R}-\mathcal{R}^{\mu}}$ is also expected to be small. 

\subsection{Flexible support vector machines}
\label{sec:equivalence}

It turns out that, for certain choices of distributions $\pp{P}$, the linear SMM trained using $\{(\pp{P}_i,y_i)\}_{i=1}^m$ is equivalent to an SVM trained using some samples $\{(x_i,y_i)\}_{i=1}^m$ with an appropriate choice of kernel function. 
 
%% lemma 
\begin{mylem}
  \label{lem:smm-svm} 
  Let $k(x,z)$ be a bounded p.d. kernel on a measure space such that $\iint k(x,z)^2\dd x\dd z < \infty$, and   $g(x,\tilde{x})$ be a square integrable function such that $\int g(x,\tilde{x})\dd\tilde{x} < \infty$ for all $x$. Given a sample $\{(\pp{P}_i,y_i)\}_{i=1}^m$ where each $\pp{P}_i$ is assumed to have a density given by $g(x_i,x)$, the linear SMM is equivalent to the SVM on the training sample $\{(x_i,y_i)\}_{i=1}^m$ with kernel $K_g(x,z)=\iint k(\tilde{x},\tilde{z})g(x,\tilde{x}) g(z,\tilde{z})\dd\tilde{x}\dd\tilde{z}$. 
\end{mylem}
%% proof 
%\begin{proof}
%For a training sample $\{(x_i,y_i)\}_{i=1}^m$, the SVM with kernel $K_g$ minimizes $\ell(\{x_i, y_i,f(x_i) + b\}_{i=1}^m) + \lambda\|f\|^2_{\hbspace_{K_g}}$. By the representer theorem, $f(x) = \sum_{i=1}^m\alpha_iK_g(x,x_j)$ with some $\alpha_i\in\mathbb{R}$, hence this is equivalent to $\ell(\{x_i,y_i,\sum_{j=1}^m\alpha_jK_g(x_i,x_j)+b\}_{i=1}^m) + \lambda\sum_{i,j=1}^m\alpha_i\alpha_jK_g(x_i,x_j)$. Next, consider the kernel mean of the probability measure $g(x_i,x)dx$ given by $\mu_i = \int k(\cdot,\tilde{x})g(x_i, \tilde{x})\dd\tilde{x}$ and note that $\langle\mu_i, f\rangle_{\hbspace_k}=\int f(\tilde{x})g(x_i, \tilde{x})\dd\tilde{x}$ for any $f\in\hbspace_k$. The linear SMM with loss $\ell$ and kernel $k$ minimizes $\ell(\{\pp{P}_i,y_i,\langle\mu_i,f \rangle_{\hbspace_k} + b\}_{i=1}^m) + \lambda\|f\|^2_{\hbspace_k}$. By Theorem \ref{thm:representer}, each minimizer $f$ admits a representation of the form $f = \sum_{j=1}^m\alpha_j\mu_j=\sum_{j=1}^m\alpha_j\int k(\cdot,\tilde{x})g(x_j, \tilde{x})\dd\tilde{x}$. Thus, for this $f$ we have $ \langle \mu_i,f \rangle_{\hbspace_k} = \sum_{j=1}^m\alpha_j\iint k(\tilde{z},\tilde{x})g(x_i,\tilde{x})g(x_j,\tilde{z}) \dd\tilde{x}\dd\tilde{z} = \sum_{j=1}^m\alpha_jK_g(x_i,x_j)$ and $\|f\|^2_{\hbspace_k} = \sum_{i,j=1}^m\alpha_i\alpha_j\langle\mu_i,\mu_j\rangle $ $ = \sum_{i,j=1}^m\alpha_i\alpha_jK_g(x_i,x_j)$, as above.
%\end{proof}

Note that the important assumption for this equivalence is that the distributions $\pp{P}_i$ differ only in their location in the parameter space. This need not be the case in all possible applications of SMMs.

Furthermore, we have $K_g(x,z)=\left\langle \int k(\tilde{x},\cdot)g(x,\tilde{x}) \dd\tilde{x} , \int k(\tilde{z},\cdot)g(z,\tilde{z})\dd\tilde{z}\right\rangle_{\hbspace}$. Thus, it is clear that the feature map of $x$ depends not only on the kernel $k$, but also on the density $g(x,\tilde{x})$. Consequently, by virtue of Lemma \ref{lem:smm-svm}, the kernel $K_g$ allows the SVM to place different kernels at each data point. We call this algorithm a \emph{flexible SVM} (Flex-SVM).

Consider  for example the linear SMM with Gaussian distributions $\mathcal{N}(x_1;\sigma^2_1\cdot\mathbf{I}),\dotsc,\mathcal{N}(x_m;\sigma^2_m\cdot\mathbf{I})$ and Gaussian RBF kernel $k_{\sigma^2}$ with bandwidth parameter $\sigma$. The convolution theorem of Gaussian distributions implies that this SMM is equivalent to a flexible SVM that places a data-dependent kernel $k_{\sigma^2+2\sigma^2_i}(x_i,\cdot)$ on training example $x_i$, i.e., a Gaussian RBF kernel with larger bandwidth.

%Figure \ref{fig:locdep-svm} illustrates the difference between a standard SVM and a flexible SVM 
%with Gaussian RBF kernel. 

%% location-dependent SVM
%\begin{figure}[t!]
%  \centering
%  \includegraphics[width=3in]{gauss.eps}
%  \caption{Standard SVM and flexible SVM with Gaussian RBF kernel and Gaussian 
%    measures on $\mathbb{R}$.}
%  \label{fig:locdep-svm}
%\end{figure} 

\section{Related works}
\label{sec:relatedworks}

The kernel $K(\pp{P},\pp{Q}) = \langle \abbrvmm{\pp{P}},\abbrvmm{\pp{Q}}\rangle_{\hbspace}$ is in fact a special case of the Hilbertian metric \cite{Hein05Hilbertian}, with the associated kernel $K(\pp{P},\pp{Q}) = \mathbb{E}_{\pp{P},\pp{Q}}[k(x,\tilde{x})]$, and a generative mean map kernel (GMMK) proposed by \cite{Mehta10Generative}. In the GMMK, the kernel between two objects $x$ and $y$ is defined via $\hat{p}_x$ and $\hat{p}_y$, which are estimated probabilistic models of $x$ and $y$, respectively. That is, a probabilistic model $\hat{p}_x$ is learned for each example and used as a surrogate to construct the kernel between those examples. The idea of surrogate kernels has also been adopted by the Probability Product Kernel (PPK) \cite{Jebara04Probabilityproduct}. In this case, we have $K_{\rho}(p,p')=\int_{\inspace}p(x)^{\rho}p'(x)^{\rho}\dd x$, which has been shown to be a special case of GMMK when $\rho=1$ \cite{Mehta10Generative}. Consequently, GMMK, PPK with $\rho=1$, and our linear kernels are equivalent when the embedding kernel is $k(x,x')=\delta(x-x')$. More recently, the empirical kernel \eqref{eq:emp-kernel} was employed in an unsupervised way for multi-task learning to generalize to a previously unseen task \cite{Blanchard11Generalize}. In contrast, we treat the probability distributions in a supervised way (cf. the regularized functional \eqref{eq:regfunc}) and the kernel is not restricted to only the empirical kernel.

The use of expected kernels in dealing with the uncertainty in the input data has a connection to robust SVMs. For instance, a generalized form of the SVM in \cite{Shivaswamy06SOCP} incorporates the probabilistic uncertainty into the maximization of the margin. This results in a second-order cone programming (SOCP) that generalizes the standard SVM. In SOCP, one needs to specify the parameter $\tau_i$ that reflects the probability of correctly classifying the $i$th training example. The parameter $\tau_i$ is therefore closely related to the parameter $\sigma_i$, which specifies the variance of the distribution centered at the $i$th example. \cite{Anderson11Missing} showed the equivalence between SVMs using expected kernels and SOCP when $\tau_i=0$. When $\tau_i > 0$, the mean and covariance of missing kernel entries have to be estimated explicitly, making the SOCP more involved for nonlinear kernels. Although achieving comparable performance to the standard SVM with expected kernels, the SOCP requires a more computationally extensive SOCP solver, as opposed to simple quadratic programming (QP).

%In practice, the expected kernel has proven successful in a wide range of applications. In 
%\cite{Gomez-chova10Cloud}, the kernel was used in semi-supervised learning for a cloud screening 
%application to reinforce the similarities between clusters of examples. Moreover, it was also applied in 
%noisy linear time-invariant system to build a classifier robust to noisy test signal 
%\cite{Anderson11Channel-robust}. Similarly, the expected kernel was employed to deal with uncertainty 
%arising from unobserved values in missing data problems \cite{Anderson11Missing}. 
 
\section{Experimental results}
\label{sec:experiments}

In the experiments, we primarily consider three different learning algorithms:
\begin{inparaenum}[i)]
\item \textbf{SVM} is considered as a baseline algorithm.
\item \textbf{Augmented SVM (ASVM)} is an SVM trained on augmented samples drawn according 
to the distributions $\{\pp{P}_i\}_{i=1}^m$. The same number of examples are drawn from each distribution.
\item \textbf{SMM} is distribution-based method that can be
applied directly on the distributions\footnote{We used the LIBSVM implementation.}.
\end{inparaenum}
  
\subsection{Synthetic data}

Firstly, we conducted a basic experiment that illustrates a fundamental difference between SVM, ASVM, and SMM. A binary classification problem of 7 Gaussian distributions with different means and covariances was considered. We trained the SVM using only the means of the distributions, ASVM with 30 virtual examples generated from each distribution, and SMM using distributions as training examples. A Gaussian RBF kernel with $\gamma = 0.25$ was used for all algorithms.

Figure \ref{fig:asvm-smm-compare} shows the resulting decision boundaries. Having been trained only on means of the distributions, the SVM classifier tends to overemphasize the regions with high densities and underrepresent the lower density regions. In contrast, the ASVM is more expensive and sensitive to outliers, especially when learning on heavy-tailed distributions. The SMM treats each distribution as a training example and implicitly incorporates properties of the distributions, i.e., means and covariances, into the classifier. Note that the SVM can be trained to achieve a similar result to the SMM by choosing an appropriate value for $\gamma$ (cf. Lemma \ref{lem:smm-svm}). Nevertheless, this becomes more difficult if the training distributions are, for example, nonisotropic and have different covariance matrices. 

Secondly, we evaluate the performance of the SMM for different combinations of embedding and level-2 kernels. Two classes of synthetic Gaussian distributions on $\mathbb{R}^{10}$ were generated. The mean parameters of the positive and negative distributions are normally distributed with means $m^+=(1,\dotsc,1)$ and $m^- = (2,\dotsc,2)$ and identical covariance matrix $\Sigma=0.5\cdot\mathbf{I}_{10}$, respectively. The covariance matrix for each distribution is generated according to two Wishart distributions with covariance matrices given by $\Sigma^+ = 0.6\cdot\mathbf{I}_{10}$ and $\Sigma^- = 1.2\cdot\mathbf{I}_{10}$ with $10$ degrees of freedom. The training set consists of 500 distributions from the positive class and 500 distributions from the negative class. The test set consists of 200 distributions with the same class proportion.

\begin{figure}[t!]
  \centering
  \begin{subfigure}[b]{0.4\linewidth}
    \includegraphics[width=2.35in]{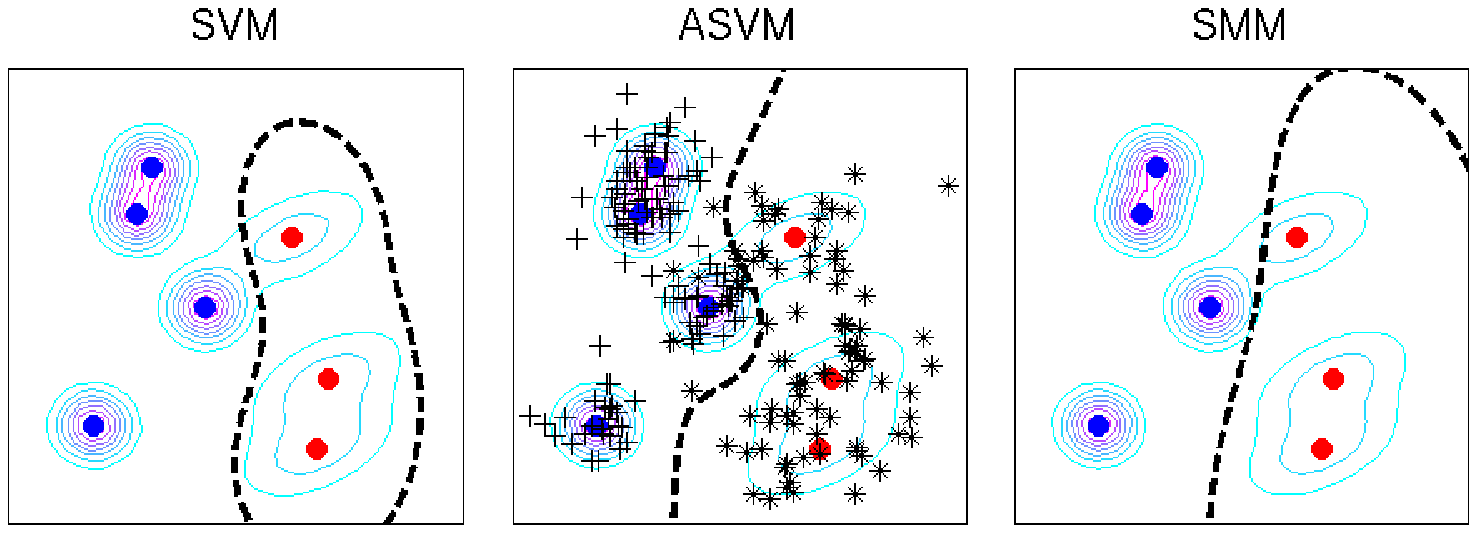} 
    \vspace{1pt}
    \caption{decision boundaries.} 
    \label{fig:asvm-smm-compare}
    %%%%%%
  \end{subfigure}
  \hfill
  \begin{subfigure}[b]{0.55\linewidth}
    \centering
    \begin{tabular}{@{}c @{}c @{}c}
      \includegraphics[width=0.9in]{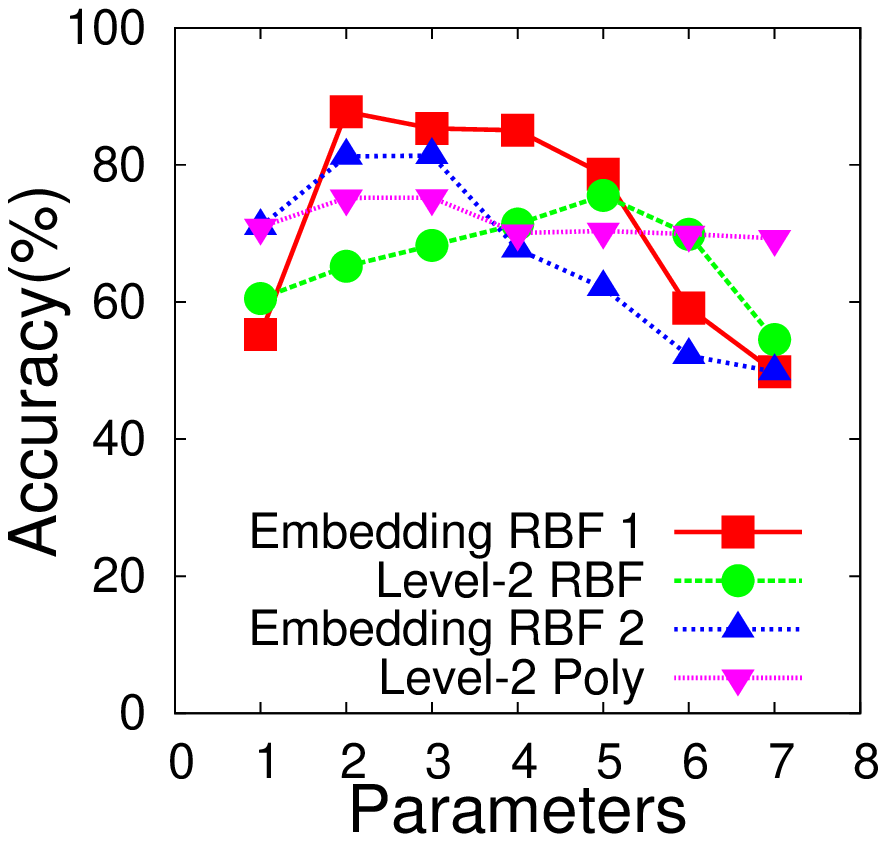} \hspace{2pt}& 
      \includegraphics[height=0.9in]{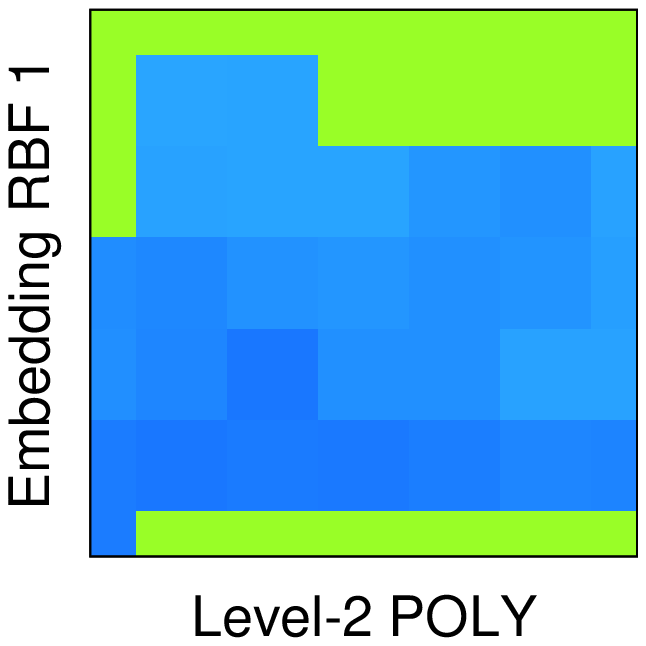} &
      \includegraphics[height=0.9in]{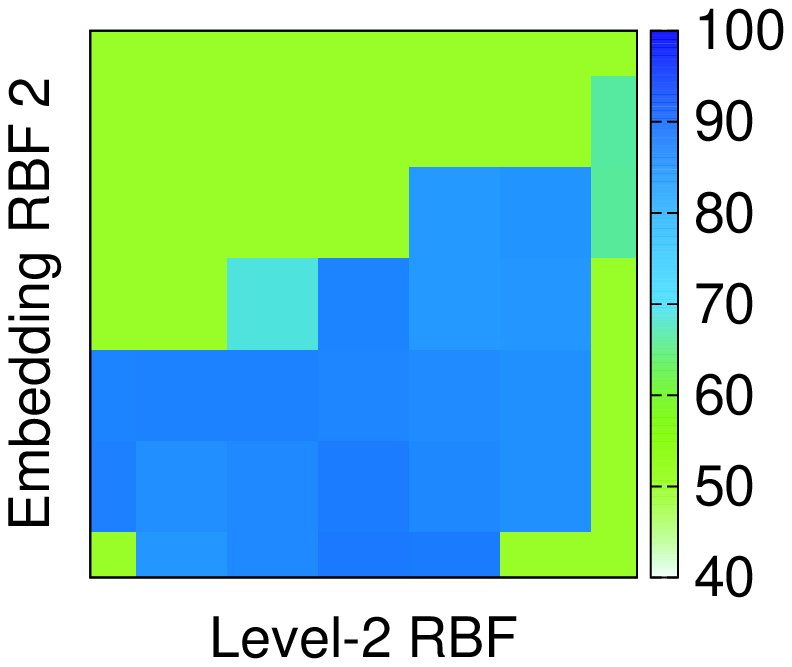}
    \end{tabular}
    \caption{sensitivity of kernel parameters} 
    \label{fig:sensitivity}     
  \end{subfigure}
  \caption{(\subref{fig:asvm-smm-compare}) the decision boundaries of SVM, ASVM, and SMM. (\subref{fig:sensitivity}) the heatmap plots of average accuracies of SMM over 30 experiments using POLY-RBF (center) and RBF-RBF (right) kernel combinations with the plots of average accuracies at different parameter values (left).}
\end{figure}

%\begin{figure}[t!]
%  \centering 
%  \includegraphics[width=2.8in]{asvm-smm-compare.eps} 
%  \caption{Decision boundaries of SVM, ASVM, and SMM.} 
%  \label{fig:asvm-smm-compare}
%\end{figure}

%\begin{table}[t!]
%  \centering
%  \caption{Accuracies (\%) of SMM on synthetic data with different combinations of embedding and level-2 
%    kernels.}
%  \resizebox{8.2cm}{!}{
%    \scriptsize{
%    \begin{tabular}{l@{}llccc}
%      \toprule
%      & & & \multicolumn{3}{c}{Level-2 kernels} \\
%      & & & LIN & POLY & RBF \\
%      \midrule
%      \multirow{5}{*}{\begin{sideways}Embedding\end{sideways}} 
%      & \multirow{5}{*}{\begin{sideways}kernels\end{sideways}}
%      & LIN & 85.20$\pm$2.20 & 83.95$\pm$2.11 & 87.80$\pm$1.96 \\
%      & & POLY2 & 81.04$\pm$3.11 & 81.34$\pm$1.21 & 73.12$\pm$3.29 \\
%      & & POLY3 & 81.10$\pm$2.76 & 82.66$\pm$1.75 & 78.28$\pm$2.19 \\
%      & & RBF & 87.74$\pm$2.19 & 88.06$\pm$1.73 & \textbf{89.65}$\pm$\textbf{1.37}\\
%      & & URBF & 85.39$\pm$2.56 & 86.84$\pm$1.51 & 86.86$\pm$1.88 \\
%      \bottomrule    
%    \end{tabular}}}    
%    \label{tab:kernelchoices}
%\end{table} 

\begin{table}[t!]
  \centering
  \caption{accuracies (\%) of SMM on synthetic data with different combinations of embedding and level-2 
    kernels.}
    \begin{tabular}{l@{}llccccc}
      \toprule
      & & & \multicolumn{5}{c}{\textbf{Embedding kernels}} \\
      & & & LIN & POLY2 & POLY3 & RBF & URBF \\
      \midrule
      \multirow{3}{*}{\begin{sideways}\textbf{Level-2}\end{sideways}} 
      & \multirow{3}{*}{\begin{sideways}\textbf{kernels}\end{sideways}}
      & LIN & 85.20$\pm$2.20 & 81.04$\pm$3.11 & 81.10$\pm$2.76 & 87.74$\pm$2.19 & 85.39$\pm$2.56 \\
      & & POLY & 83.95$\pm$2.11 & 81.34$\pm$1.21 & 82.66$\pm$1.75 & 88.06$\pm$1.73 & 86.84$\pm$1.51 \\
      & & RBF & 87.80$\pm$1.96 & 73.12$\pm$3.29 & 78.28$\pm$2.19 & \textbf{89.65}$\pm$\textbf{1.37} & 86.86$\pm$1.88 \\
      \bottomrule    
    \end{tabular}
    \label{tab:kernelchoices}
\end{table}

The kernels used in the experiment include linear kernel (LIN), polynomial kernel of degree 2 (POLY2), polynomial kernel of degree 3 (POLY3), unnormalized Gaussian RBF kernel (RBF), and normalized Gaussian RBF kernel (NRBF). To fix parameter values of both kernel functions and SMM, 10-fold cross-validation (10-CV) is performed on a parameter grid, $C\in\{2^{-3},2^{-2},\dotsc,2^{7}\}$ for SMM, bandwidth parameter $\gamma\in\{10^{-3}, 10^{-2},\dotsc,10^{2}\}$ for Gaussian RBF kernels, and degree parameter $d\in\{2,3,4,5,6\}$ for polynomial kernels. The average accuracy and $\pm 1$ standard deviation for all kernel combinations over 30 repetitions are reported in Table \ref{tab:kernelchoices}. Moreover, we also investigate the sensitivity of kernel parameters for two kernel combinations: RBF-RBF and POLY-RBF. In this case, we consider the bandwidth parameter $\gamma = \{10^{-3},10^{-2},\dotsc,10^{3}\}$ for Gaussian RBF kernels and degree parameter $d=\{2,3,\dotsc,8\}$ for polynomial kernels. Figure \ref{fig:sensitivity} depicts the accuracy values and average accuracies for considered kernel functions. 

%\begin{figure}[t!]   
%  \centering
%  \begin{tabular}{@{}c @{}c @{}c}
%    \includegraphics[width=1in]{sensitivity-avg.eps} \hspace{2pt}& 
%    \includegraphics[height=1in]{sensitivity-poly.eps} &
%    \includegraphics[height=1in]{sensitivity-rbf.eps}
%  \end{tabular}
%  \caption{The heatmap plots of average accuracies of SMM over 30 experiments using POLY-RBF (center) and RBF-RBF (right) kernel combinations with the plots of average accuracies at different parameter values (left).} 
%  \label{fig:sensitivity} 
%\end{figure} 

%% results
\begin{figure*}[t!]
  \centering
  \resizebox{\linewidth}{!}{
  \begin{minipage}[b]{0.57\linewidth}
    \centering
    \includegraphics[width=2.0in]{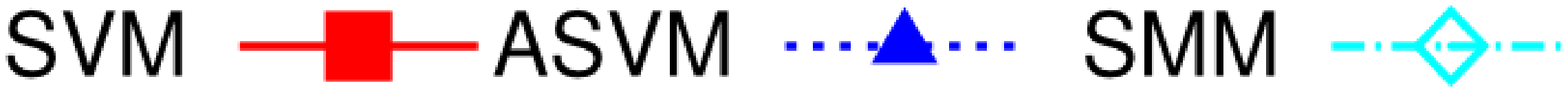}
    \includegraphics[width=3.3in]{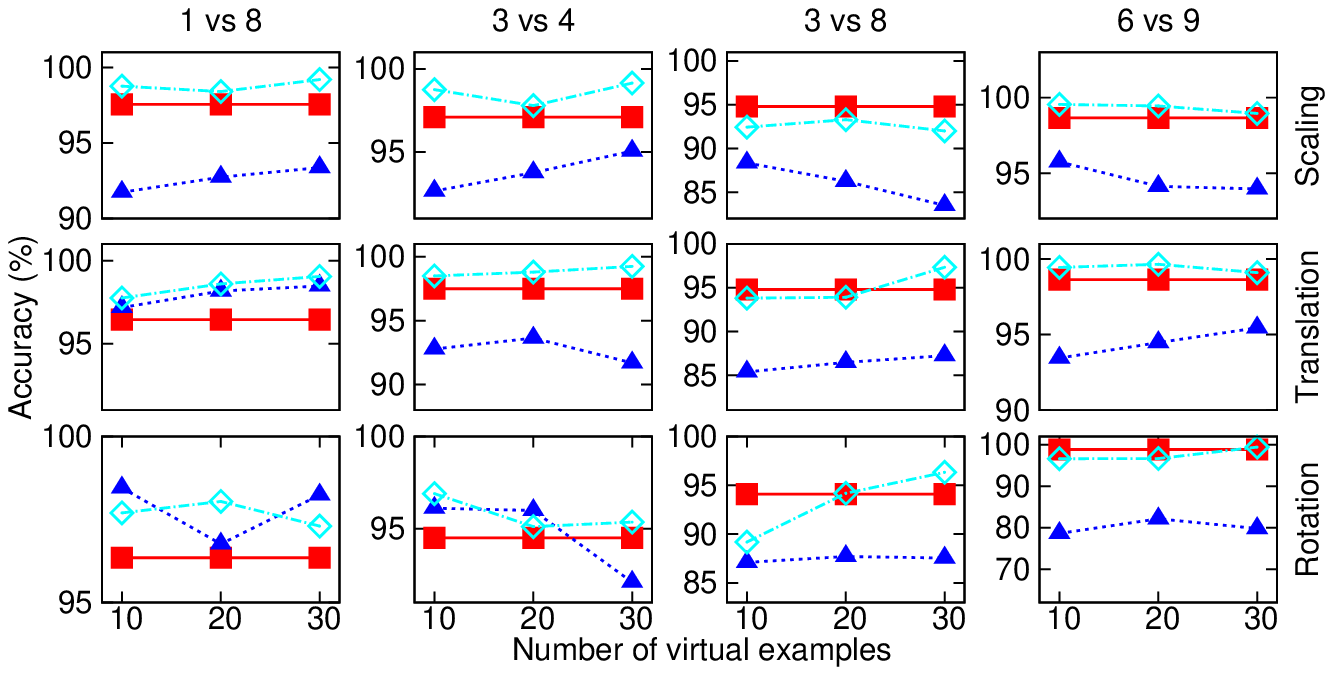}
    \caption{the performance of SVM, ASVM, and SMM algorithms on handwritten digits constructed using three 
      basic transformations.}
    \label{fig:usps-invariant}
  \end{minipage} 
  \hspace{0.6cm}
  \begin{minipage}[b]{0.4\linewidth}
    \centering
    \includegraphics[width=1.4in]{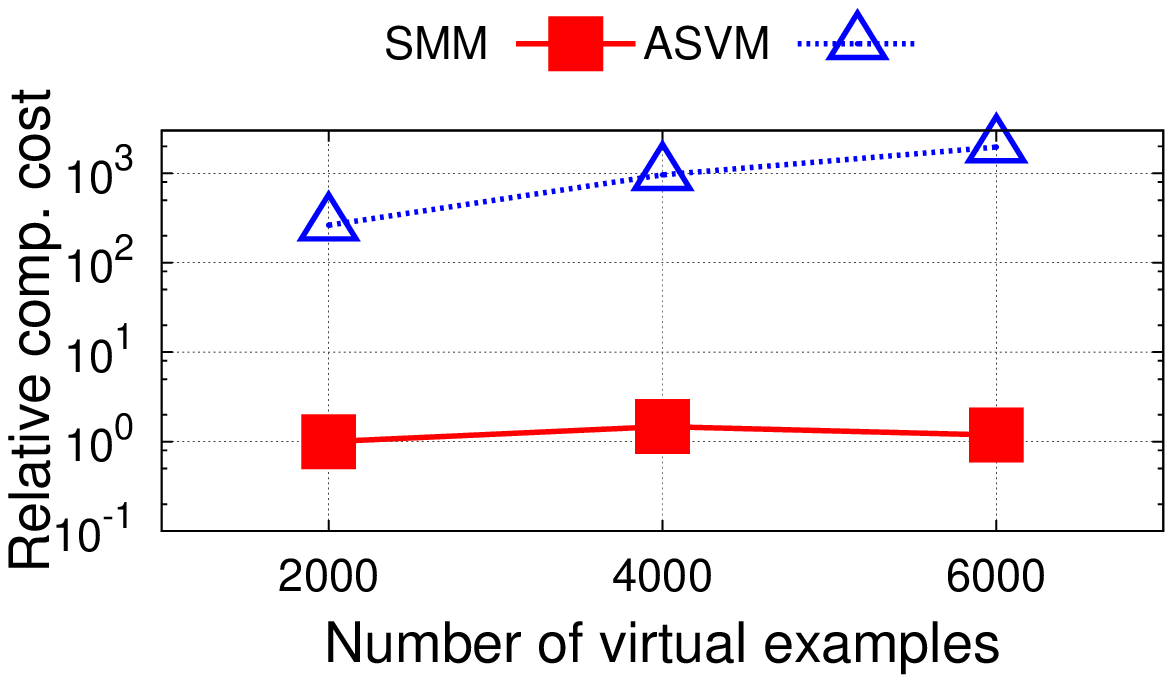} 
    \vspace{-10pt}
    \caption{relative computational cost of ASVM and SMM (baseline: SMM with 2000 virtual examples).} 
    \label{fig:usps-invariant-time} 
    \vspace{5pt}
    \includegraphics[width=1.4in]{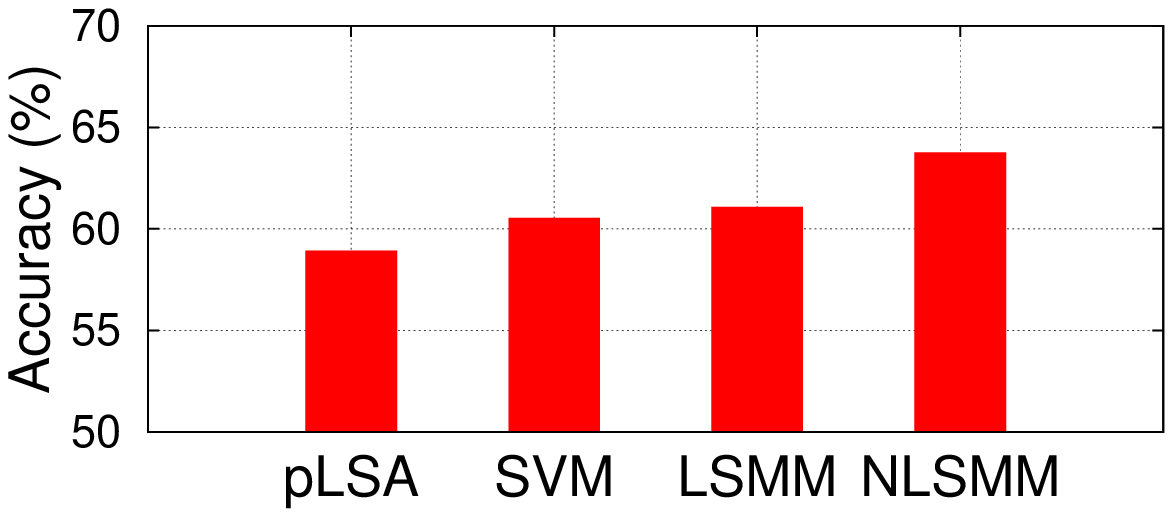}
    \vspace{-10pt}
    \caption{accuracies of four different techniques for natural scene categorization.}
    \label{fig:naturalscene-acc}
  \end{minipage} }
\end{figure*} 
  
Table \ref{tab:kernelchoices} indicates that both embedding and level-2 kernels are important for the performance of the classifier. The embedding kernels tend to have more impact on the predictive performance compared to the level-2 kernels. This conclusion also coincides with the results depicted in Figure \ref{fig:sensitivity}.  
 
%To understand this, note that the RKHSs $\hbspace$ and $\hbspace_{K}$ induced by the embedding kernel $k$ 
%and level-2 kernel $K$ play different roles. While $\hbspace$ serves as a representation space, i.e., 
%feature space, for distributions, $\hbspace_K$ serves as a function class on distributions defined using 
%mean embeddings. Therefore, choosing the parameter value for the embedding kernel can be regarded as 
%finding a suitable representation for the input distributions. Once one has found a suitable 
%representation, the learning problem becomes much simpler.

\subsection{Handwritten digit recognition} 
\label{sec:handwritten}
 
In this section, the proposed framework is applied to distributions over equivalence classes of images that are invariant to basic transformations, namely, \emph{scaling}, \emph{translation}, and \emph{rotation}. We consider the handwritten digits obtained from the USPS dataset. For each $16\times 16$ 
image, the distribution over the equivalence class of the transformations is determined by a prior on parameters associated with such transformations. Scaling and translation are parametrized by the scale factors $(s_x,s_y)$ and displacements $(t_x,t_y)$ along the $x$ and $y$ axes, respectively. The rotation is parametrized by an angle $\theta$. We adopt Gaussian distributions as prior distributions, including $\mathcal{N}([1,1],0.1\cdot\mathbf{I}_2)$, $\mathcal{N}([0,0],5\cdot\mathbf{I}_2)$, and $\mathcal{N}(0;\pi)$. For each image, the virtual examples are obtained by sampling parameter values from the distribution and applying the transformation accordingly.

Experiments are categorized into simple and difficult binary classification tasks. The former consists of classifying digit 1 against digit 8 and digit 3 against digit 4. The latter considers classifying digit 3 against digit 8 and digit 6 against digit 9. The initial dataset for each task is constructed by randomly selecting 100 examples from each class. Then, for each example in the initial dataset, we generate 10, 20, and 30 virtual examples using the aforementioned transformations to construct virtual data sets consisting of 2,000, 4,000, and 6,000 examples, respectively. One third of examples in the initial dataset are used as a test set. The original examples are excluded from the virtual datasets. The virtual examples are normalized such that their feature values are in $[0,1]$. Then, to reduce computational cost, principle component analysis (PCA) is performed to reduce the dimensionality to 16. We compare the SVM on the initial dataset, the ASVM on the virtual datasets, and the SMM. For SVM and ASVM, the Gaussian RBF kernel is used. For SMM, we employ the empirical kernel \eqref{eq:emp-kernel} with Gaussian RBF kernel as a base kernel. The parameters of the algorithms are fixed by 10-CV over parameters $C\in\{2^{-3},2^{-2},\dotsc,2^7\}$ and $\gamma\in\{0.01,0.1,1\}$.

The results depicted in Figure \ref{fig:usps-invariant} clearly demonstrate the benefits of learning directly from the equivalence classes of digits under basic transformations\footnote{While the reported results were obtained using virtual examples with Gaussian parameter distributions (Sec. \ref{sec:handwritten}), we got similar results using uniform distributions.}. In most cases, the SMM outperforms both the SVM and the ASVM as the number of virtual examples increases. Moreover, Figure \ref{fig:usps-invariant-time} shows the benefit of the SMM over the ASVM in term of computational cost\footnote{The evaluation was made on a 64-bit desktop computer with Intel$^{\text{\textregistered}}$ Core$^{\text{\texttrademark}}$ 2 Duo CPU E8400 at 3.00GHz$\times$2 and 4GB of memory.}.
 
\subsection{Natural scene categorization}
 
This section illustrates benefits of the nonlinear kernels between distributions for learning natural scene categories in which the bag-of-word (BoW) representation is used to represent images in the dataset. Each image is represented as a collection of local patches, each being a codeword from a large vocabulary of codewords called codebook. Standard BoW representations encode each image as a histogram that enumerates the occurrence probability of local patches detected in the image w.r.t. those in the codebook. On the other hand, our setting represents each image as a distribution over these codewords. Thus, images of different scenes tends to generate distinct set of patches. Based on this representation, both the histogram and the local patches can be used in our framework.

We use the dataset presented in \cite{Fei-fei05:BHM}. According to their results, most errors occurs among the four indoor categories (830 images), namely, bedroom (174 images), living room (289 images), kitchen (151 images), and office (216 images). Therefore, we will focus on these four categories. For each category, we split the dataset randomly into two separate sets of images, 100 for training and the rest for testing.

A codebook is formed from the training images of all categories. Firstly, interesting keypoints in the image are randomly detected. Local patches are then generated accordingly. After patch detection, each patch is transformed into a 128-dim SIFT vector \cite{Lowe99:SIFT}. Given the collection of detected patches, K-means clustering is performed over all local patches. Codewords are then defined as the centers of the learned clusters. Then, each patch in an image is mapped to a codeword and the image can be represented by the histogram of the codewords. In addition, we also have an $M\times 128$ matrix of SIFT vectors where $M$ is the number of codewords.

We compare the performance of a Probabilistic Latent Semantic Analysis (pLSA) with the standard BoW representation, SVM, linear SMM (LSMM), and nonlinear SMM (NLSMM). For SMM, we use the empirical embedding kernel with Gaussian RBF base kernel $k$: $K(\mathbf{h}_i,\mathbf{h}_j) = \sum_{r=1}^M\sum_{s=1}^M h_i(c_r)h_j(c_s)k(c_r,c_s)$ where $\mathbf{h}_i$ is the histogram of the $i$th image and $c_r$ is the $r$th SIFT vector. A Gaussian RBF kernel is also used as the level-2 kernel for nonlinear SMM. For the SVM, we adopt a Gaussian RBF kernel with $\chi^2$-distance between the histograms \cite{Vedaldi09:multiple}, i.e., $K(\mathbf{h}_i,\mathbf{h}_j)= \exp\left(-\gamma\chi^2(\mathbf{h}_i,\mathbf{h}_j)\right)$ where $\chi^2(\mathbf{h}_i,\mathbf{h}_j) = \sum_{r=1}^M\frac{(h_i(c_r)-h_j(c_r))^2}{h_i(c_r)+h_j(c_r)}$. The parameters of the algorithms are fixed by 10-CV over parameters $C\in\{2^{-3},2^{-2},\dotsc,2^7\}$ and $\gamma\in\{0.01,0.1,1\}$. For NLSMM, we use the best $\gamma$ of LSMM in the base kernel and perform 10-CV to choose $\gamma$ parameter only for the level-2 kernel. To deal with multiple categories, we adopt the pairwise approach and voting scheme to categorize test images. The results in Figure \ref{fig:naturalscene-acc} illustrate the benefit of the distribution-based framework. Understanding the context of a complex scene is challenging. Employing distribution-based methods provides an elegant way of utilizing higher-order statistics in natural images that could not be captured by traditional sample-based methods.
 
\section{Conclusions}
\label{sec:conclusions}
 
This paper proposes a method for kernel-based discriminative learning on probability distributions. The trick is to embed distributions into an RKHS, resulting in a simple and efficient learning algorithm on distributions. A family of linear and nonlinear kernels on distributions allows one to flexibly choose the kernel function that is suitable for the problems at hand. Our analyses provide insights into the relations between distribution-based methods and traditional sample-based methods, particularly the flexible SVM that allows the SVM to place different kernels on each training example. The experimental results illustrate the benefits of learning from a pool of distributions, compared to a pool of examples, both on synthetic and real-world data. 

\subsubsection*{Acknowledgments}   

KM would like to thank Zoubin Gharamani, Arthur Gretton, Christian Walder, and Philipp Hennig for a fruitful discussion. We also thank all three insightful reviewers for their invaluable comments. 
   
\small{
\bibliography{smm-nips2012}  

\begin{thebibliography}{10}

\bibitem{Yang02:cDNA}
Y.~H. Yang and T.~Speed.
\newblock Design issues for {cDNA} microarray experiments.
\newblock {\em Nat. Rev. Genet.}, 3(8):579--588, 2002.

\bibitem{Jebara04Probabilityproduct}
T.~Jebara, R.~Kondor, A.~Howard, K.~Bennett, and N.~Cesa-bianchi.
\newblock Probability product kernels.
\newblock {\em Journal of Machine Learning Research}, 5:819--844, 2004.

\bibitem{Bhattacharyya43Kernel}
A.~Bhattacharyya.
\newblock On a measure of divergence between two statistical populations
  defined by their probability distributions.
\newblock {\em Bull. Calcutta Math Soc.}, 1943.

\bibitem{Moreno04KL}
P.~J. Moreno, P.~P. Ho, and N.~Vasconcelos.
\newblock A {K}ullback-{L}eibler divergence based kernel for {SVM}
  classification in multimedia applications.
\newblock In {\em Proceedings of Advances in Neural Information Processing
  Systems}. MIT Press, 2004.

\bibitem{Hein05Hilbertian}
M.~Hein and O.~Bousquet.
\newblock Hilbertian metrics and positive definite kernels on probability.
\newblock In {\em Proceedings of The 12th International Conference on
  Artificial Intelligence and Statistics}, pages 136--143, 2005.

\bibitem{Cuturi05SKM}
M.~Cuturi, K.~Fukumizu, and J-P. Vert.
\newblock Semigroup kernels on measures.
\newblock {\em Journal of Machine Learning Research}, 6:1169--1198, 2005.

\bibitem{Martins09NIT}
Andr\'{e} F.~T. Martins, Noah~A. Smith, Eric~P. Xing, Pedro M.~Q. Aguiar, and
  M\'{a}rio A.~T. Figueiredo.
\newblock Nonextensive information theoretic kernels on measures.
\newblock {\em Journal of Machine Learning Research}, 10:935--975, 2009.

\bibitem{Bertinet04:RKHS}
A.~Berlinet and Thomas~C. Agnan.
\newblock {\em Reproducing kernel Hilbert spaces in probability and
  statistics}.
\newblock Kluwer Academic Publishers, 2004.

\bibitem{Smola07Hilbert}
A.~Smola, A.~Gretton, L.~Song, and B.~Sch\"{o}lkopf.
\newblock A hilbert space embedding for distributions.
\newblock In {\em Proceedings of the 18th International Conference on
  Algorithmic Learning Theory}, pages 13--31. Springer-Verlag, 2007.

\bibitem{Sriperumbudur10:Metrics}
B.~K. Sriperumbudur, A.~Gretton, K.~Fukumizu, B.~Sch\"{o}lkopf, and Gert R.~G.
  Lanckriet.
\newblock Hilbert space embeddings and metrics on probability measures.
\newblock {\em Journal of Machine Learning Research}, 99:1517--1561, 2010.

\bibitem{Schoelkopf01:Representer}
B.~Sch\"{o}lkopf, R.~Herbrich, and A.~J. Smola.
\newblock A generalized representer theorem.
\newblock In {\em COLT '01/EuroCOLT '01}, pages 416--426. Springer-Verlag,
  2001.

\bibitem{Dinuzzo12:representer}
F.~Dinuzzo and B.~Sch\"{o}lkopf.
\newblock The representer theorem for {H}ilbert spaces: a necessary and
  sufficient condition.
\newblock In {\em Advances in Neural Information Processing Systems 25}, pages
  189--196. 2012.

\bibitem{Steinwart01universal}
I.~Steinwart.
\newblock On the influence of the kernel on the consistency of support vector
  machines.
\newblock {\em Journal of Machine Learning Research}, 2:67--93, 2001.

\bibitem{Christmann10Universal}
A.~Christmann and I.~Steinwart.
\newblock Universal kernels on non-standard input spaces.
\newblock In {\em Proceedings of Advances in Neural Information Processing
  Systems}, pages 406--414. 2010.

\bibitem{Mehta10Generative}
N.~A. Mehta and A.~G. Gray.
\newblock Generative and latent mean map kernels.
\newblock {\em CoRR}, abs/1005.0188, 2010.

\bibitem{Blanchard11Generalize}
G.~Blanchard, G.~Lee, and C.~Scott.
\newblock Generalizing from several related classification tasks to a new
  unlabeled sample.
\newblock In {\em Advances in Neural Information Processing Systems 24}, pages
  2178--2186. 2011.

\bibitem{Shivaswamy06SOCP}
P.~K. Shivaswamy, C.~Bhattacharyya, and A.~J. Smola.
\newblock Second order cone programming approaches for handling missing and
  uncertain data.
\newblock {\em Journal of Machine Learning Research}, 7:1283--1314, 2006.

\bibitem{Anderson11Missing}
H.S. Anderson and M.R. Gupta.
\newblock Expected kernel for missing features in support vector machines.
\newblock In {\em Statistical Signal Processing Workshop}, pages 285--288,
  2011.

\bibitem{Fei-fei05:BHM}
L.~Fei-fei.
\newblock A bayesian hierarchical model for learning natural scene categories.
\newblock In {\em Proceedings of the {IEEE} Conference on Computer Vision and
  Pattern Recognition ({CVPR})}, pages 524--531, 2005.

\bibitem{Lowe99:SIFT}
D.~G. Lowe.
\newblock Object recognition from local scale-invariant features.
\newblock In {\em Proceedings of the International Conference on Computer
  Vision}, pages 1150--1157, Washington, DC, USA, 1999.

\bibitem{Vedaldi09:multiple}
A.~Vedaldi, V.~Gulshan, M.~Varma, and A.~Zisserman.
\newblock Multiple kernels for object detection.
\newblock In {\em Proceedings of the International Conference on Computer
  Vision}, pages 606--613, 2009.

\end{thebibliography}
\bibliographystyle{unsrt}}
 
%%%%%%%%%%%%%%%%%%%%%%%%%%%%%%%%%%%%%%%%%%%%%%%%%%%%%%
\newpage
\appendix
\appendixpage
\section{Proof of Theorem \ref{thm:representer}}

%% representer theorem for probability measures
\setcounter{mydef}{0}
\begin{mythm}
\label{thm:representer1}
Given training examples $(\pp{P}_i,y_i)\in \pspace\times\mathbb{R},\,i=1,\dotsc,m$, a strictly monotonically increasing function $\Omega:[0,+\infty)\rightarrow\mathbb{R}$, and a loss function $\ell:(\pspace\times\mathbb{R}^2)^m \rightarrow\mathbb{R}\cup\{+\infty\}$, any $f\in\mathcal{H}$ minimizing the regularized risk functional
\begin{equation}
  \label{eq:regfunc} 
  \ell\left(\pp{P}_1,y_1,\mathbb{E}_{\pp{P}_1}[f],\dotsc,\pp{P}_m,y_m,\mathbb{E}_{\pp{P}_m}[f]\right) 
  + \Omega\left(\|f\|_{\mathcal{H}}\right)
\end{equation}
\noindent admits a representation of the form $ f = \sum_{i=1}^m\alpha_i \abbrvmm{\pp{P}_i}$ for some $\alpha_i\in\mathbb{R}, \, i=1,\dotsc,m$.
\end{mythm}
 
%% proof of representer theorem
\begin{proof}
  By virtue of Proposition 2 in \cite{Sriperumbudur10:Metrics}, the linear functional $\mathbb{E}_{\pp{P}}[\cdot]$ are bounded for all $\pp{P}\in\pspace$. Then, given $\pp{P}_1,\pp{P}_2,...,\pp{P}_m$, any $f\in\mathcal{H}$ can be decomposed as 
  \begin{equation*}
    f = f_{\mu} + f^{\perp}
  \end{equation*}
\noindent where $f_{\mu}\in\mathcal{H}$ lives in the span of $\abbrvmm{\pp{P}_i}$, i.e., $f_{\mu}=\sum_{i=1}^m\alpha_i\abbrvmm{\pp{P}_i}$ and $f^{\perp}\in\mathcal{H}$ satisfying, for all $j$, $\langle f^{\perp},\abbrvmm{\pp{P}_j}\rangle = 0$. Hence, for all $j$, we have 
\begin{equation*}
\mathbb{E}_{\pp{P}_j}[f] = \mathbb{E}_{\pp{P}_j}[f_{\mu} + f^{\perp}] = \langle f_{\mu}+f^{\perp},\abbrvmm{\pp{P}_j}\rangle = \langle f_{\mu},\abbrvmm{\pp{P}_j}\rangle + \langle f^{\perp},\abbrvmm{\pp{P}_j}\rangle = \langle f_{\mu},\abbrvmm{\pp{P}_j}\rangle
\end{equation*}
\noindent which is independent of $f^{\perp}$. As a result, the loss functional $\ell$ in \eqref{eq:regfunc} does not depend on $f^{\perp}$. For the regularization functional $\Omega$, since $f^{\perp}$ is orthogonal to $\sum_{i=1}^m\alpha_i\abbrvmm{\pp{P}_i}$ and $\Omega$ is strictly monotonically increasing, we have 
\begin{equation*}
  \Omega(\|f\|) = \Omega(\|f_{\mu} + f^{\perp}\|)=\Omega(\sqrt{\|f_{\mu}\|^2 + \|f^{\perp}\|^2})\geq \Omega(\|f_{\mu}\|)
\end{equation*}
\noindent with equality if and only if $f^{\perp}=0$ and thus $f=f_{\mu}$. Consequently, any minimizer must take the form $f=\sum_{i=1}^m\alpha_i\abbrvmm{\pp{P}_i} = \sum_{i=1}^m\alpha_i\mathbb{E}_{\pp{P}_i}[k(x,\cdot)]$. 
\end{proof}

\section{Proof of Theorem \ref{thm:deviation}}

%% the risk deviation bound
\setcounter{mydef}{2}
\begin{mythm}
  \label{thm:deviation2}
  Given an arbitrary probability distribution $\pp{P}$ with variance $\sigma^2$, a Lipschitz continuous function $f:\mathbb{R}\rightarrow\mathbb{R}$ with constant $C_f$, an arbitrary loss function $\ell : \mathbb{R}\times\mathbb{R}\rightarrow\mathbb{R}$ that is Lipschitz continuous in the second argument with constant $C_{\ell}$, it follows that 
  \begin{equation*}
    \abs{\mathbb{E}_{x\sim\pp{P}}[\ell(y,f(x))]  - \ell(y,\mathbb{E}_{x\sim\pp{P}}[f(x)])} \leq 2C_{\ell}C_f\sigma
  \end{equation*}
\noindent for any $y\in\mathbb{R}$.
\end{mythm}
  
%% proof 
\begin{proof}
Assume that $x$ is distributed according to $\pp{P}$. Let $m_X$ be the mean of $X$ in $\mathbb{R}^d$. Thus, we have
\begin{eqnarray*}
\abs{\mathbb{E}_{\pp{P}}[\ell(y,f(x))] - \ell(y,\mathbb{E}_{\pp{P}}[f(x)])} 
&\leq& \int \abs{\ell(y,f(\tilde{x})) - \ell(y,\mathbb{E}_{\pp{P}}[f(x)])} \dd\pp{P}(\tilde{x}) \\
&\leq& C_{\ell} \int \abs{f(\tilde{x}) - \mathbb{E}_{\pp{P}}[f(x)]} \dd\pp{P}(\tilde{x}) \\
&\leq& \underbrace{C_{\ell} \int \abs{f(\tilde{x}) - f(m_X)} \dd\pp{P}(\tilde{x})}_{A} + \underbrace{C_{\ell}\abs{f(m_X) - \mathbb{E}_{\pp{P}}[f(x)]}}_{B} \; .
\end{eqnarray*}
\begin{paragraph}{Control of ($A$)}
  The first term is upper bounded by 
  \begin{equation}
    C_{\ell}\int C_f\norm{\tilde{x} - m_X} \dd\pp{P}(\tilde{x}) \leq C_{\ell}C_f\sigma \enspace ,
  \end{equation}
  \label{eq:control-A}
  \noindent where the last inequality is given by $\mathbb{E}_{\pp{P}}[\norm{\tilde{x} - m_X}] \leq \sqrt{\mathbb{E}_{\pp{P}}[\norm{\tilde{x} - m_X}^2]} = \sigma$.
\end{paragraph}

\begin{paragraph}{Control of ($B$)}
  Similarly, the second term is upper bounded by
  \begin{equation}
    \label{eq:control-B}
    C_{\ell}\left|\int f(m_X) - f(\tilde{x})\right| \dd\pp{P}(\tilde{x}) \leq C_{\ell}\int C_f\norm{m_X - \tilde{x}} \dd\pp{P}(\tilde{x}) \leq C_{\ell}C_f\sigma \enspace .
  \end{equation}
\end{paragraph}
Combining \eqref{eq:control-A} and \eqref{eq:control-B} yields
\begin{equation*}
  \abs{\mathbb{E}_{\pp{P}}[\ell(y,f(x))] - \ell(y,\mathbb{E}_{\pp{P}}[f(x)])} \leq 2C_{\ell}C_f\sigma \enspace ,
\end{equation*}
thus completing the proof.
\end{proof}

\section{Proof of Lemma \ref{lem:smm-svm}}
 
%% lemma 
\setcounter{mydef}{3}
\begin{mylem}
  \label{lem:smm-svm3} 
  Let $k(x,z)$ be a bounded p.d. kernel on a measure space such that $\iint k(x,z)^2\dd x\dd z < \infty$, and   $g(x,\tilde{x})$ be a square integrable function such that $\int g(x,\tilde{x})\dd\tilde{x} < \infty$ for all $x$. Given a sample $\{(\pp{P}_i,y_i)\}_{i=1}^m$ where each $\pp{P}_i$ is assumed to have a density given by $g(x_i,x)$, the linear SMM is equivalent to the SVM on the training sample $\{(x_i,y_i)\}_{i=1}^m$ with kernel $K_g(x,z)=\iint k(\tilde{x},\tilde{z})g(x,\tilde{x}) g(z,\tilde{z})\dd\tilde{x}\dd\tilde{z}$. 
\end{mylem}
%% proof 
\begin{proof}
For a training sample $\{(x_i,y_i)\}_{i=1}^m$, the SVM with kernel $K_g$ minimizes 
\begin{equation*}
\ell(\{x_i, y_i,f(x_i) + b\}_{i=1}^m) + \lambda\|f\|^2_{\hbspace_{K_g}} \enspace .
\end{equation*}
By the representer theorem, $f(x) = \sum_{i=1}^m\alpha_iK_g(x,x_j)$ with some $\alpha_i\in\mathbb{R}$, hence this is equivalent to 
\begin{equation*}
\ell(\{x_i,y_i,\sum_{j=1}^m\alpha_jK_g(x_i,x_j)+b\}_{i=1}^m) + \lambda\sum_{i,j=1}^m\alpha_i\alpha_jK_g(x_i,x_j) \enspace . 
\end{equation*}
Next, consider the kernel mean of the probability measure $g(x_i,x)dx$ given by $\mu_i = \int k(\cdot,\tilde{x})g(x_i, \tilde{x})\dd\tilde{x}$ and note that $\langle\mu_i, f\rangle_{\hbspace_k}=\int f(\tilde{x})g(x_i, \tilde{x})\dd\tilde{x}$ for any $f\in\hbspace_k$. The linear SMM with loss $\ell$ and kernel $k$ minimizes 
\begin{equation*}
\ell(\{\pp{P}_i,y_i,\langle\mu_i,f \rangle_{\hbspace_k} + b\}_{i=1}^m) + \lambda\|f\|^2_{\hbspace_k} \enspace . 
\end{equation*}
By Theorem \ref{thm:representer}, each minimizer $f$ admits a representation of the form 
\begin{equation*}
f = \sum_{j=1}^m\alpha_j\mu_j=\sum_{j=1}^m\alpha_j\int k(\cdot,\tilde{x})g(x_j, \tilde{x})\dd\tilde{x} \enspace . 
\end{equation*}
Thus, for this $f$ we have 
\begin{equation*}
 \langle \mu_i,f \rangle_{\hbspace_k} = \sum_{j=1}^m\alpha_j\iint k(\tilde{z},\tilde{x})g(x_i,\tilde{x})g(x_j,\tilde{z}) \dd\tilde{x}\dd\tilde{z} = \sum_{j=1}^m\alpha_jK_g(x_i,x_j) 
\end{equation*}
\noindent and 
\begin{equation*}
\|f\|^2_{\hbspace_k} = \sum_{i,j=1}^m\alpha_i\alpha_j\langle\mu_i,\mu_j\rangle = \sum_{i,j=1}^m\alpha_i\alpha_jK_g(x_i,x_j)
\end{equation*}
\noindent , as above. This completes the proof.
\end{proof}

\end{document}